%% file: iclr2027_conference.tex
\documentclass{article} 
\usepackage{iclr2027_conference,times}

\usepackage[table]{xcolor}
\usepackage{url}
\usepackage{amsmath,amssymb,amsthm}
\usepackage{booktabs,multirow,graphicx,wrapfig}
\usepackage{subcaption}
\usepackage{algorithm}
\usepackage{algpseudocode}
\usepackage{hyperref}

\definecolor{mainblue}{HTML}{248FD6}      
\colorlet{highlight}{mainblue!10}         
\colorlet{algorange}{highlight}           

\hypersetup{
  colorlinks=true,
  linkcolor=mainblue,
  citecolor=mainblue,
  urlcolor=mainblue
}

\theoremstyle{plain}
\newtheorem{theorem}{Theorem}
\newtheorem{lemma}[theorem]{Lemma}

\renewcommand{\headrulewidth}{0pt}

\title{Preserving Updates in Low-Bit Recurrent State Quantization}

\author{%
{Haotian Li\thanks{Equal contribution.}},\hspace{0.5em}
{Liran Dong\footnotemark[1]},\hspace{0.5em}
Kaicheng Xiao,\hspace{0.5em}
Guoliang Xing\\
The Chinese University of Hong Kong\\
\texttt{\{lh023, dl123, xk023, glxing\}@ie.cuhk.edu.hk}
}

\iclrfinalcopy
\begin{document}

\maketitle
\lhead{}

\begin{abstract}
Linear attention and recurrent models compress past tokens into a fixed-size
state that is updated at every step.
However, storing these states in high precision incurs substantial memory and
bandwidth costs at large batch sizes, motivating low-bit state quantization.
Unlike weights and key--value caches, a recurrent state is requantized after
every update, so its quantization errors accumulate across steps.
Our analysis shows that these errors are highly coherent across steps, so
reducing per-step errors alone is insufficient.
This coherence arises mainly from small updates that are lost in rounding.
To preserve these updates, we intervene both before and at rounding.
Before rounding, we propose Low-Rank Bypass (LRB), which accumulates small
updates in a compact low-rank buffer and periodically merges its sufficiently
large elements into the state.
At rounding, we observe that a state element can even move opposite to its
update, and such reversals prove particularly harmful.
Therefore, we propose Bounded Directional Rounding (BDR), which rounds reversed
elements toward their updates without additional storage.
Although BDR enlarges the rounding error, it improves the model, suggesting that
the direction of an update matters more than its rounding error at a single
step.
Across Gated DeltaNet, GLA, Mamba2, and RetNet, our training-free method
substantially outperforms INT4, while reducing state memory by up to
5.3$\times$ and speeding up decoding by up to 1.9$\times$ relative to FP32
states.
\end{abstract}

\input{sec/1_intro}

\input{sec/2_bkg}

\input{sec/3_lrb}

\input{sec/5_eval}

\input{sec/6_related}

\input{sec/7_discussion}

\bibliography{iclr2027_conference}
\bibliographystyle{iclr2027_conference}

\appendix

\input{sec/appendix_proof}

\end{document}

%% file: sec/1_intro.tex
\section{Introduction}

Language modeling requires storing contextual information and retrieving relevant associations as new tokens arrive.
Transformers provide this capability through softmax attention, but they keep the keys and values of all previous tokens, so their memory and computation grow with the context length~\citep{vaswani2017attention}.
To avoid this growth, linear attention and related recurrent models compress past information into a fixed-size state, so that computation grows linearly and memory stays constant with the context length~\citep{katharopoulos2020transformers,schlag2021linear,dao2024transformers}.
Owing to this efficiency, these architectures have been adopted in recent large language models~\citep{chandiramani2026nemotron}.

However, since the fixed-size state accumulates the updates of all previous tokens, it is commonly stored in FP32 to preserve their precision~\citep{chandiramani2026nemotron}. With such a state maintained for each sequence in every recurrent layer, the total state memory grows linearly with the batch size and limits the number of concurrent sequences~\citep{tianqi-etal-2025-q}. For example, at a batch size of 128, the recurrent states of a 1.3B Mamba2 model~\citep{dao2024transformers} occupy about 12~GiB. Moreover, every decoding step reads and writes the entire state, so the state accounts for a large share of memory traffic and limits decoding throughput at large batch sizes.

A natural way to reduce these costs is to store the states in lower
precision~\citep{tianqi-etal-2025-q,chandiramani2026nemotron}.
Quantization has been widely studied for weights, activations, and key--value
caches.
Existing methods suppress outliers through scaling, rotation, or per-channel
quantization~\citep{xiao2023smoothquant,lin2024awq,shao2024omniquant,ashkboos2024quarot,liu2024kivi,hooper2024kvquant},
or minimize the layer-wise reconstruction error with second-order
information~\citep{frantar2022gptq}.
Since these tensors are quantized only once, these methods aim to reduce the
error of a single quantization.
A recurrent state, however, is requantized after every update, so its accuracy
depends on how the errors of all steps accumulate.

We analyze this accumulation and find that the errors of different steps are
highly coherent (Sec.~\ref{sec:error}).
Instead of cancelling out like random noise, they add up across steps.
Therefore, reducing the error of each quantization alone is insufficient.
For example, a Hadamard rotation suppresses outliers to reduce the per-step
error, but the perplexity at long contexts remains far above that of FP32
states (Sec.~\ref{sec:exp}).
To find the source of this coherence, we examine how rounding treats a single
update.
The quantization grid is fitted to a state that accumulates many past updates,
so its spacing is much larger than a single update.
As a result, rounding often leaves the state unchanged, and the update is lost.
We further find that the per-step errors consist largely of these lost updates,
which makes them coherent across steps.

To keep small updates from being lost, we intervene both before and at
rounding.
Before rounding, we propose Low-Rank Bypass (LRB), which accumulates small
updates in a bypass outside the quantized state and periodically merges its
sufficiently large elements into the state.
LRB exploits the outer-product structure of each update to keep the bypass in
low-rank form, which adds little extra storage.
At rounding, we observe that a state element can even move opposite to its
update, and such reversals turn out to be particularly harmful to the model.
We therefore propose Bounded Directional Rounding (BDR), which rounds a reversed
element toward its update without additional storage.
Counterintuitively, BDR increases the rounding error but still improves the
model, suggesting that the direction of an update matters more than its rounding
error at a single step.
Our contributions are summarized as follows.
\begin{itemize}
    \vspace{-0.1cm}
    \item An analysis of error accumulation in low-bit recurrent states.
    We show that the per-step errors add up coherently across steps, so reducing
    the error of each quantization is insufficient, and we trace this coherence
    to small updates lost in rounding.
    \item A training-free state quantization method, in which LRB preserves the
    magnitude of small updates before rounding and BDR preserves their
    direction at rounding.
    We demonstrate the method on Gated DeltaNet~\citep{yang2025gated},
    GLA~\citep{yang2024gated}, Mamba2~\citep{dao2024transformers}, and
    RetNet~\citep{sun2023retentive}.
    It substantially improves model quality over INT4, while reducing state
    memory by 4.7--5.3$\times$ and speeding up decoding by up to 1.9$\times$
    over FP32 states.
\end{itemize}

%% file: sec/2_bkg.tex
\vspace{-0.2cm}
\section{Error Accumulation in Quantized Recurrent States}
\label{sec:error}

\subsection{Quantized Recurrent States}
\label{sec:2-1}

\begin{wraptable}{r}{0.57\textwidth}
\vspace{-\intextsep}
\centering
\small
\setlength{\tabcolsep}{6pt}
\renewcommand{\arraystretch}{1.2}
\begin{tabular}{lll}
\toprule
Model & Decay $\mathbf A_t$ & Update $\mathbf B_t$ \\
\midrule
RetNet & $\gamma\mathbf I$ & $\mathbf k_t\mathbf v_t^\top$ \\
Mamba2 & $\alpha_t\mathbf I$ & $\mathbf k_t\mathbf v_t^\top$ \\
GLA & $\operatorname{Diag}(\boldsymbol\alpha_t)$ & $\mathbf k_t\mathbf v_t^\top$ \\
GDN & $\alpha_t\mathbf I$ & $\beta_t\mathbf k_t\bigl(\mathbf v_t-(\mathbf A_t\mathbf S_{t-1})^\top\mathbf k_t\bigr)^\top$ \\
\bottomrule
\end{tabular}
\vspace{-0.1cm}
\caption{Decay and update terms of the state recurrence.}
\vspace{-0.2cm}
\label{tab:la_models}
\end{wraptable}

Linear attention and recurrent models compress past key--value associations
into a fixed-size state $\mathbf S_t\in\mathbb R^{d_k\times d_v}$, keeping state
storage independent of context
length~\citep{katharopoulos2020transformers,schlag2021linear}.
At each step, a key $\mathbf k_t\in\mathbb R^{d_k}$ and a value
$\mathbf v_t\in\mathbb R^{d_v}$ update the state as
$\mathbf S_t=\mathbf A_t\mathbf S_{t-1}+\mathbf B_t$, and a query
$\mathbf q_t\in\mathbb R^{d_k}$ reads it as $\mathbf o_t=\mathbf S_t^\top\mathbf q_t$.
Here, $\mathbf A_t$ controls the decay of historical information, and
$\mathbf B_t$ is the update written by the current token.
Table~\ref{tab:la_models} lists both terms for the models we study.
Specifically, the decay is a scalar multiple of the identity in
RetNet~\citep{sun2023retentive}, Mamba2~\citep{dao2024transformers}, and
Gated DeltaNet (GDN)~\citep{yang2025gated}, and a diagonal matrix in
GLA~\citep{yang2024gated}.
Moreover, every update in Table~\ref{tab:la_models} is an outer product of
the key and a written value.

Executing this recurrence requires reading and writing the entire state at
every decoding step.
Consequently, both storage and memory traffic grow with the batch size, which
motivates low-precision state storage~\citep{tianqi-etal-2025-q}.
A common remedy is to compute updates in full precision while storing the
updated state in low precision~\citep{chandiramani2026nemotron}.
We follow this setting and use group-wise asymmetric
quantization~\citep{nagel2021white}, which rounds each element to the nearest
point of a uniform grid.
The grid spacing $\Delta_t$ is determined by the range of each group and is
refitted at every step (details in App.~\ref{app:quant}).
The quantized state $\widehat{\mathbf S}_t$ then evolves as
\begin{equation}
\begin{aligned}
\widehat{\mathbf S}_t
&=\mathcal Q_t\bigl(\mathbf A_t\widehat{\mathbf S}_{t-1}+\mathbf B_t\bigr)
=\mathbf A_t\widehat{\mathbf S}_{t-1}+\widetilde{\mathbf B}_t,\\
\mathbf R_t
&=\mathcal Q_t\bigl(\mathbf A_t\widehat{\mathbf S}_{t-1}+\mathbf B_t\bigr)
-\bigl(\mathbf A_t\widehat{\mathbf S}_{t-1}+\mathbf B_t\bigr)
=\widetilde{\mathbf B}_t-\mathbf B_t,
\end{aligned}
\label{eq:quant}
\end{equation}
where $\widetilde{\mathbf B}_t$ is the effective update written into the
quantized state, and $\mathbf R_t$ is the per-step error introduced by rounding.
As a result, quantization replaces each intended update $\mathbf B_t$ with
$\widetilde{\mathbf B}_t$.
Moreover, the per-step error $\mathbf R_t$ is carried into all subsequent states
and thus affects later outputs.

\begin{figure}
    \centering
    \includegraphics[width=1\linewidth]{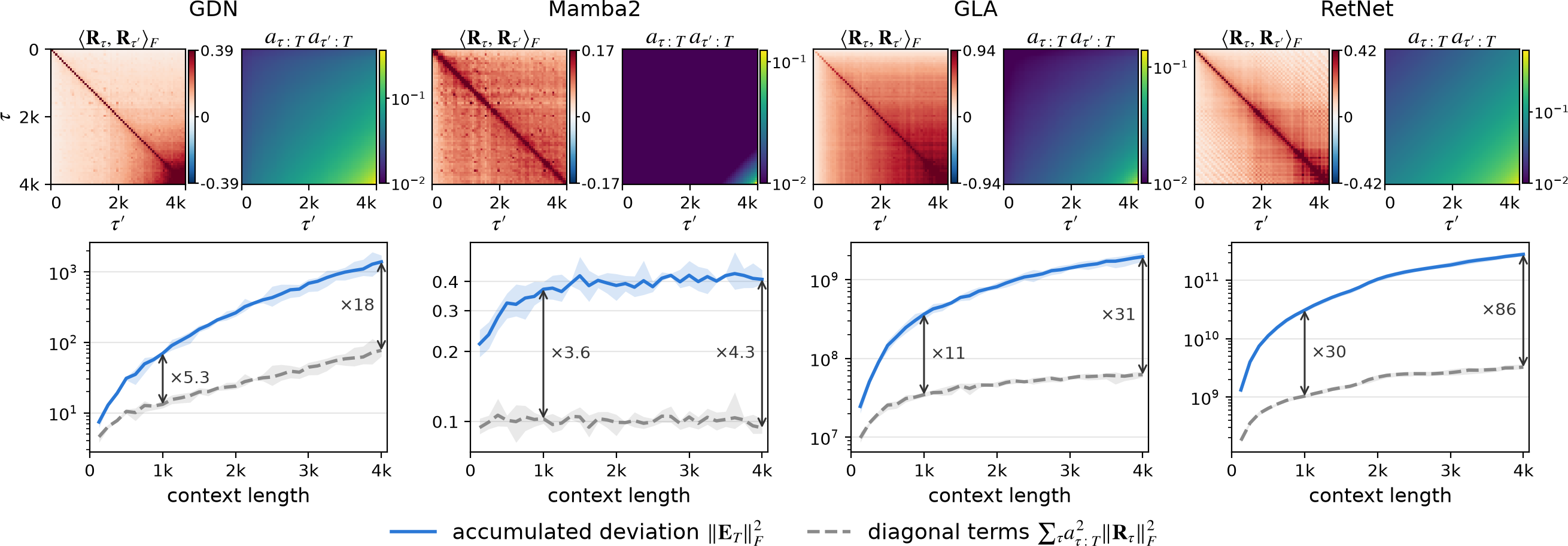}
    \vspace{-0.5cm}
\caption{Accumulation of per-step errors under 4-bit state quantization. Top: inner products $\langle \mathbf{R}_\tau, \mathbf{R}_{\tau'} \rangle_F$ (left), summed over 64-step blocks and normalized by the mean diagonal entry for each layer and head before averaging, and decay weights $a_{\tau:T}\, a_{\tau':T}$ at $T = 4$K (right). Bottom: $\|\mathbf{E}_T\|_F^2$ and its diagonal terms as the context grows, with their ratio $g_T$ marked by arrows.}
    \label{fig:coherence}
    \vspace{-0.3cm}
\end{figure}

\subsection{Coherent Accumulation of Per-Step Errors}
\label{sec:2-2}
\vspace{-0.1cm}

Quantization methods for weights, activations, and key--value caches reduce
the error of each quantization
independently~\citep{frantar2022gptq,lin2024awq,xiao2023smoothquant,ashkboos2024quarot,liu2024kivi,hooper2024kvquant}.
This suffices because each of these tensors is quantized only once.
In contrast, the recurrent state is updated and quantized again at every step,
so its per-step errors $\mathbf R_t$ accumulate.
As a result, the degradation grows with context length.
For example, under 4-bit quantization, the WikiText-103 \citep{merity2016pointer} perplexity of RetNet
rises from 59.2 at 1K tokens to 1193.2 at 8K, whereas that of FP32 states stays
below 26 (Table~\ref{tab:wiki_ppl}).
To understand this growth, we unroll both recurrences with the same updates $\mathbf{B}_\tau$. With $\mathbf{A}_{\tau:T} = \prod_{j=\tau+1}^{T} \mathbf{A}_j$ denoting the decay from step $\tau$ to step $T$, the accumulated deviation of the quantized state from the full-precision state is
\begin{equation}
\mathbf E_T=\widehat{\mathbf S}_T-\mathbf S_T
=\sum_{\tau\le T}\mathbf A_{\tau:T}\widetilde{\mathbf B}_\tau
-\sum_{\tau\le T}\mathbf A_{\tau:T}\mathbf B_\tau
=\sum_{\tau\le T}\mathbf A_{\tau:T}\mathbf R_\tau .
\label{eq:acc_dev}
\end{equation}

For RetNet, Mamba2, and Gated DeltaNet, whose decay is scalar with
$\mathbf A_{\tau:T}=a_{\tau:T}\mathbf I$, the squared Frobenius norm of this
deviation expands as
\begin{equation}
\|\mathbf E_T\|_F^2
=\sum_{\tau\le T}a_{\tau:T}^2\,\|\mathbf R_\tau\|_F^2
+\sum_{\tau\neq\tau'}a_{\tau:T}\,a_{\tau':T}\,
\langle\mathbf R_\tau,\mathbf R_{\tau'}\rangle_F .
\label{eq:dev_split}
\end{equation}
For GLA, whose decay differs across rows, the same expansion holds for each
row of the state.
The first sum contains the diagonal terms with $\tau=\tau'$, which depend
only on the norm of each per-step error.
Conventional quantization methods aim to minimize the error of each
quantization and therefore account only for these terms.
In contrast, the second sum contains the off-diagonal terms between
different steps.
Each of them is the inner product of two per-step errors, weighted by the
decay weight $a_{\tau:T}\,a_{\tau':T}$ that measures how much of the two
errors the state retains at step $T$.
The inner product is positive when the two errors point in similar
directions.

In Eq.~\ref{eq:dev_split}, each term of $\|\mathbf E_T\|_F^2$ is the product
of two factors: the inner product $\langle\mathbf R_\tau,\mathbf R_{\tau'}\rangle_F$
of the per-step errors and the decay weight $a_{\tau:T}\,a_{\tau':T}$.
Fig.~\ref{fig:coherence} (top) plots both factors on WikiText-103
sequences.
First, the inner products are positive not only on the diagonal but also off the diagonal for nearly all pairs of steps, even for steps several thousand tokens apart.
Second, the decay weights determine how many past steps the state retains.
They remain large over several thousand tokens for GDN, GLA, and
RetNet, but only over the last few hundred tokens for Mamba2.
The number of diagonal terms grows linearly, whereas the number of off-diagonal terms grows quadratically.
Hence, the stronger the off-diagonal correlation and the slower the decay,
the more the positive off-diagonal terms dominate $\|\mathbf E_T\|_F^2$, even
though each of them is smaller than a diagonal term.

To quantify this effect, the coherent gain
$g_T=\|\mathbf E_T\|_F^2\big/\sum_{\tau\le T}a_{\tau:T}^2\|\mathbf R_\tau\|_F^2$
compares the squared Frobenius norm of the accumulated deviation with that
of its diagonal terms.
If the errors of different steps were uncorrelated, $g_T$ would remain close
to $1$.
Instead, Fig.~\ref{fig:coherence} (bottom) shows that the accumulated
deviation far exceeds its diagonal terms on all four models, and the gap
widens with context to up to $86\times$ at 4K tokens.
Consistent with the decay weights, the gain of Mamba2 grows only slightly
after the first thousand tokens and stays around $4$, whereas that of RetNet
grows the most.
Perplexity under 4-bit quantization follows the same trend: on WikiText-103
at 4K tokens, it increases by $8.6$ for Mamba2 but by $157.1$ for RetNet
(Table~\ref{tab:wiki_ppl}).
Consequently, \textbf{the errors of different steps combine coherently, and
the resulting off-diagonal terms dominate the accumulated deviation},
especially in models that decay slowly.

\subsection{Coherence Arises from Lost Updates}
\label{sec:2-3}
\vspace{-0.1cm}

\begin{wrapfigure}{R}{0.45\textwidth}
\vspace{-\intextsep}
\centering
\includegraphics[width=\linewidth]{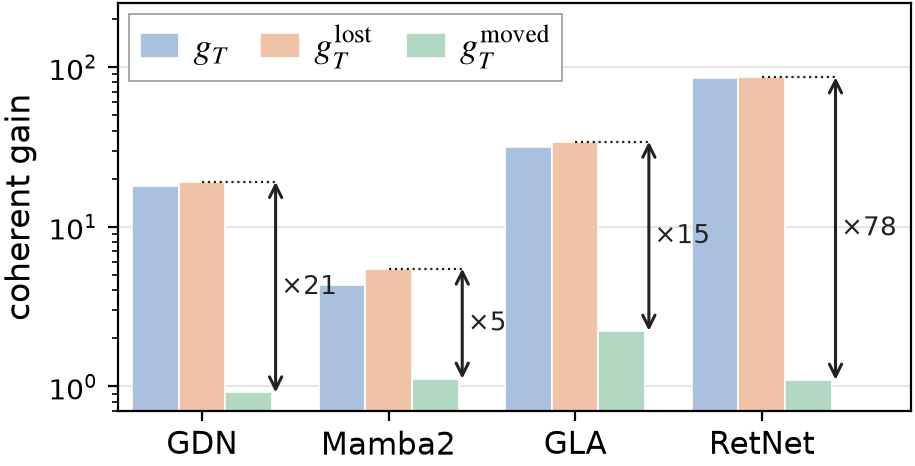}
\vspace{-0.3cm}
\caption{Coherent gains of all per-step errors ($g_T$) and of the errors on
elements whose update is lost ($g^{\text{lost}}_T$) or moved to a different grid
point ($g^{\text{moved}}_T$). Arrows mark $g^{\text{lost}}_T/g^{\text{moved}}_T$.}
\label{fig:coh_comp}
\vspace{-0.35\intextsep}
\end{wrapfigure}

To identify the source of the coherence in Sec.~\ref{sec:2-2}, we examine how
each update is rounded.
The grid spacing $\Delta_t$ is fitted to the range of each group, whose elements
accumulate many past updates.
A single update $\mathbf B_t$ is therefore typically much smaller than
$\Delta_t$ in each element.
Adding such an update often does not change the rounding result, so the
quantized state is the same as if the update had never been written.
In this case, we say that the update is lost in rounding.
We mark the elements with lost updates by the element-wise indicator
$\mathbf M_t=\mathbf 1\bigl[\mathcal Q_t(\mathbf A_t\widehat{\mathbf S}_{t-1}+\mathbf B_t)
=\mathcal Q_t(\mathbf A_t\widehat{\mathbf S}_{t-1})\bigr]$,
where both sides are rounded onto the same grid.
Such losses are common: under 4-bit quantization, the median ratio of a nonzero
update element to its grid spacing is only 0.02--0.21 across the four models, and
68--96\% of these elements are lost at each step (details in App.~\ref{app:lost_stats}).
On these elements, Eq.~\ref{eq:quant} becomes
\begin{equation}
\mathbf R_t
=\bigl(\mathcal Q_t(\mathbf A_t\widehat{\mathbf S}_{t-1})-\mathbf A_t\widehat{\mathbf S}_{t-1}\bigr)
-\mathbf B_t.
\label{eq:lost_error}
\end{equation}
The first term is the rounding error of the decayed state and does not directly
depend on $\mathbf B_t$, whereas the second term is the update itself with the
opposite sign.
The error thus contains the entire update that quantization fails to write.
Consequently, errors on elements with lost updates are expected to inherit the
correlation of the updates across steps.

To test this, we split each per-step error into
$\mathbf R_t^{\mathrm{lost}}=\mathbf M_t\odot\mathbf R_t$ and
$\mathbf R_t^{\mathrm{moved}}=(\mathbf 1-\mathbf M_t)\odot\mathbf R_t$, where the
latter covers the elements whose update changes the rounding result.
We then measure the coherent gain of each part in the same way as $g_T$:
\begin{equation}
g_T^{\mathrm{lost}}
=\frac{\bigl\|\sum_{\tau\le T}a_{\tau:T}\mathbf R_\tau^{\mathrm{lost}}\bigr\|_F^2}
{\sum_{\tau\le T}a_{\tau:T}^2\,\|\mathbf R_\tau^{\mathrm{lost}}\|_F^2},
\qquad
g_T^{\mathrm{moved}}
=\frac{\bigl\|\sum_{\tau\le T}a_{\tau:T}\mathbf R_\tau^{\mathrm{moved}}\bigr\|_F^2}
{\sum_{\tau\le T}a_{\tau:T}^2\,\|\mathbf R_\tau^{\mathrm{moved}}\|_F^2}.
\label{eq:gain_split}
\end{equation}
Since each gain is normalized by its own diagonal terms, the two gains remain
comparable even though the two parts cover different numbers of elements.
Fig.~\ref{fig:coh_comp} shows a sharp contrast between the two parts on
WikiText-103 sequences at $T=4$K.
On the elements with lost updates, $g_T^{\mathrm{lost}}$ is comparable to $g_T$ on all four models.
In contrast, $g_T^{\text{moved}}$ stays close to 1 (at most about 2), so the errors on the remaining elements are only weakly correlated across steps.
Consequently, \textbf{the coherence of the per-step errors arises mainly from the
small updates lost in rounding.}
Reducing the accumulated deviation therefore calls for writing these small
updates into the state, rather than reducing per-step errors alone.

%% file: sec/3_lrb.tex
\section{Preserving Updates under Low-Bit Storage}
Sec.~\ref{sec:2-3} shows that the harmful part of the accumulated deviation
comes from small updates lost in rounding.
These updates are lost because each of them is rounded at a single step onto a
grid much coarser than itself.
We therefore preserve them at two stages of writing an update into the state:
before rounding and at rounding.
Before rounding, we propose Low-Rank Bypass (LRB), which changes \emph{what} is
rounded.
It accumulates small updates in a compact low-rank buffer and periodically
merges its sufficiently large elements into the state, which preserves their
magnitude at a small storage cost (Sec.~\ref{sec:lrb}).
At rounding, we propose Bounded Directional Rounding (BDR), which changes
\emph{how} an update is rounded.
When rounding would move an element against its update, BDR instead rounds it
toward the update, which preserves the direction of the update without
additional storage (Sec.~\ref{sec:bdr}).
Since the two methods act at different stages, each can be applied on its own,
and combining them gives the best results (App.~\ref{app:ablation}).

\begin{figure}
    \centering
    \includegraphics[width=1\linewidth]{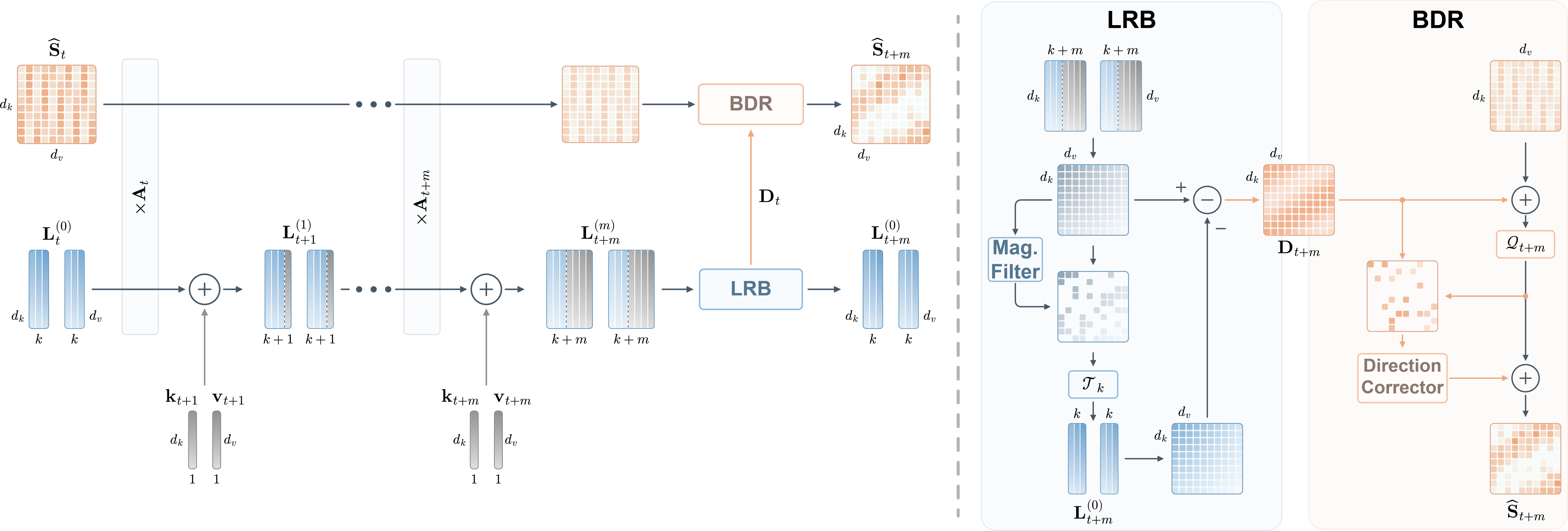}
    \caption{Overview of our method. LRB accumulates updates in a low-rank bypass
and writes them into the quantized state every $m$ steps, and BDR corrects the
reversed elements during rounding.}
    \label{fig:overview}
    \vspace{-0.4cm}
\end{figure}

\subsection{LRB: Preserving the Magnitude of Small Updates}
\label{sec:lrb}

Accumulating updates outside the quantized state requires extra storage, which
would amount to a full-precision copy of the state if kept densely.
To store them compactly, LRB exploits the structure of the updates.
Each update is the outer product $\mathbf k_t\bar{\mathbf v}_t^\top$ of a key and
a written value\footnote{$\bar{\mathbf v}_t=\beta_t\bigl(\mathbf v_t-(\mathbf A_t\mathbf S_{t-1})^\top\mathbf k_t\bigr)$
for GDN and $\bar{\mathbf v}_t=\mathbf v_t$ for the other models
(Table~\ref{tab:la_models}).},
so the sum of several updates has rank at most the number of updates.
LRB therefore keeps them in a low-rank bypass
$\mathbf L_t=\mathbf K_t\bar{\mathbf V}_t^\top$, where $\mathbf K_t$ and
$\bar{\mathbf V}_t$ collect their keys and written values as columns.
Every $m$ steps, LRB merges the bypass into the quantized state.
Between two merges, the quantized state only decays, and each new update is
appended to the bypass:
\begin{equation}
\widehat{\mathbf S}_t=\mathbf A_{t_0:t}\widehat{\mathbf S}_{t_0},\qquad
\mathbf L_t=\mathbf A_t\mathbf L_{t-1}+\mathbf k_t\bar{\mathbf v}_t^\top
=\bigl[\mathbf A_t\mathbf K_{t-1},\,\mathbf k_t\bigr]
 \bigl[\bar{\mathbf V}_{t-1},\,\bar{\mathbf v}_t\bigr]^\top,
\label{eq:lrb}
\end{equation}
where $t_0$ is the most recent merge step.
We store $\widehat{\mathbf S}_{t_0}$ unchanged together with the decay
$\mathbf A_{t_0:t}$ accumulated since then, so the stored state is not rounded
again until the next merge.
The last equality shows that the bypass is updated through its factors, by
decaying the keys and appending one column to each factor.
A query also reads the bypass through its factors:
$\mathbf o_t=(\widehat{\mathbf S}_t+\mathbf L_t)^\top\mathbf q_t
=\widehat{\mathbf S}_t^\top\mathbf q_t+\bar{\mathbf V}_t\bigl(\mathbf K_t^\top\mathbf q_t\bigr)$.
Hence $\mathbf L_t$ is never materialized between merges.

At a merge, some accumulated updates are still small, and merging them loses
them in rounding.
A merge should therefore hold back the elements marked by $\mathbf M_t$
(Sec.~\ref{sec:2-3}), whose updates are lost in rounding, and leave them in the
bypass to accumulate until a later merge.
However, $\mathbf M_t$ is not available before the merge: it relies on the
grid, which is fitted only after the merged elements are chosen.
We therefore approximate $\mathbf M_t$ by the magnitude of each accumulated
update, and keep in the bypass the elements marked by
$\mathbf M^L_t=\mathbf 1\bigl[\,|\mathbf L_t|<\epsilon_L\Delta_t\,\bigr]$,
where $\Delta_t$ is the grid spacing of
$\mathbf A_{t_0:t}\widehat{\mathbf S}_{t_0}+\mathbf L_t$.
Since the kept elements no longer form a few outer products, we retain only
their rank-$k$ approximation, denoted by $\mathcal T_k$, to keep the bypass in
factored form, and merge the residual into the state.
A merge at step $t$ thus computes
\begin{equation}
\mathbf L_t'=\mathcal T_k\bigl(\mathbf L_t\odot\mathbf M^L_t\bigr),\qquad
\mathbf D_t=\mathbf L_t-\mathbf L_t',\qquad
\widehat{\mathbf S}_t=\mathcal Q_t\bigl(\mathbf A_{t_0:t}\widehat{\mathbf S}_{t_0}
+\mathbf D_t\bigr),
\label{eq:merge}
\end{equation}
where $\mathbf D_t$ is the update written into the quantized state, after which
$t_0$ is set to $t$.
We compute $\mathcal T_k$ with a randomized SVD~\citep{halko2011finding}, which
finds the top-$k$ components from a few random projections instead of a full
SVD (implementation details in App.~\ref{app:impl}).
It returns $\mathbf L_t'$ in factored form, so the bypass restarts from $k$
columns after each merge and holds at most $k+m$ columns, i.e.,
$(k+m)(d_k+d_v)$ elements.
Fig.~\ref{fig:overview} illustrates the steps between two merges, where
$\mathbf L^{(j)}$ denotes the bypass after $j$ updates are appended since the
last merge, so that $\mathbf L^{(0)}=\mathbf L'$.

\subsection{BDR: Preserving the Direction of Lost Updates}
\label{sec:bdr}
Beyond whether an update is lost (Sec.~\ref{sec:2-3}), how it is rounded also
determines the update actually written into the state.
Consider a single element at a merge (Eq.~\ref{eq:merge}), and let $s$, $d$,
and $\widetilde d$ denote its elements in
$\mathbf A_{t_0:t}\widehat{\mathbf S}_{t_0}$, $\mathbf D_t$, and the effective
update $\widetilde{\mathbf D}_t=\widehat{\mathbf S}_t-\mathbf A_{t_0:t}\widehat{\mathbf S}_{t_0}$.
The effective update then splits into two parts:
\begin{equation}
\widetilde d
=\mathcal Q_t(s+d)-s
=\underbrace{\mathcal Q_t(s+d)-\mathcal Q_t(s)}_{\text{rounded update}}
+\underbrace{\mathcal Q_t(s)-s}_{\text{re-rounding shift}}.
\label{eq:split}
\end{equation}
The two parts relate to $d$ differently.
The rounded update keeps the sign of $d$, since it is either zero or a multiple
of $\Delta_t$ with the sign of $d$.
In contrast, the re-rounding shift does not follow $d$.
Since the grid is refitted at every rounding, $s$ generally no longer lies on
the grid and is moved to its nearest grid point, which can lie on either side
of $s$.
The shift can therefore have the opposite sign of $d$.
Consequently, the effective update deviates from $d$ in one of three ways, as
illustrated in Fig.~\ref{fig:correction} (left):
(i) \emph{undershoot}, where $\widetilde d$ keeps the sign of $d$ but
$|\widetilde d|<|d|$;
(ii) \emph{overshoot}, where $\widetilde d$ keeps the sign of $d$ but
$|\widetilde d|>|d|$; and
(iii) \emph{reversal}, where $\widetilde d$ has the opposite sign of $d$.
A reversal arises only from a small update that is lost in rounding, with
$|d|<\Delta_t/2$ (Lemma~\ref{lem:correction}\,(i)).

\paragraph{Reversals are the most harmful.}
To compare the three cases, we remove the rounding errors of each case
separately under 4-bit quantization and measure the change in perplexity.
At every step, we set $\widetilde d=d$ on a random fraction of the elements in
one case.
As Fig.~\ref{fig:correction} (right) shows, removing the errors of reversals
reduces perplexity the most on all four models at the same fraction.
In contrast, removing the errors of overshoots helps less, and removing those
of undershoots does not help.
This agrees with Sec.~\ref{sec:2-3}.
When an update is lost, the rounded update in Eq.~\ref{eq:split} vanishes, so
the error consists of $-d$ and the re-rounding shift.
The part $-d$ accumulates coherently, and in a reversal the re-rounding shift
points the same way and further enlarges it.

\begin{figure}[t]
  \centering
  \includegraphics[height=3.6cm]{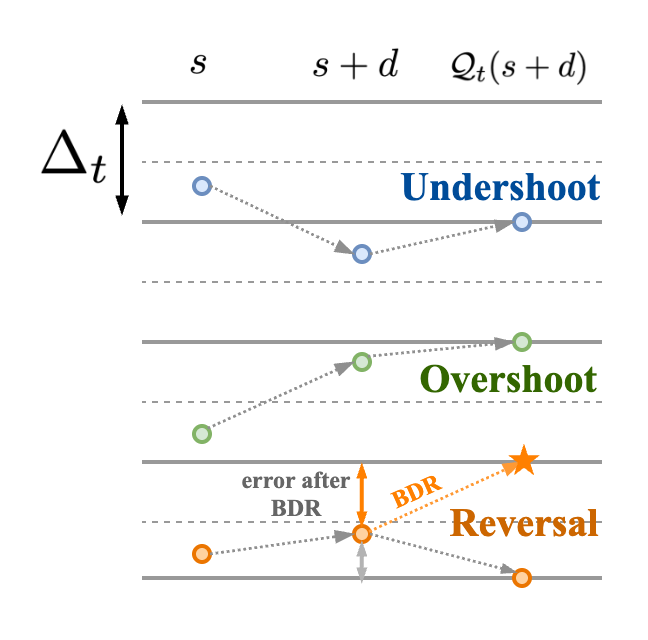}
  \hfill
  \includegraphics[height=3.6cm]{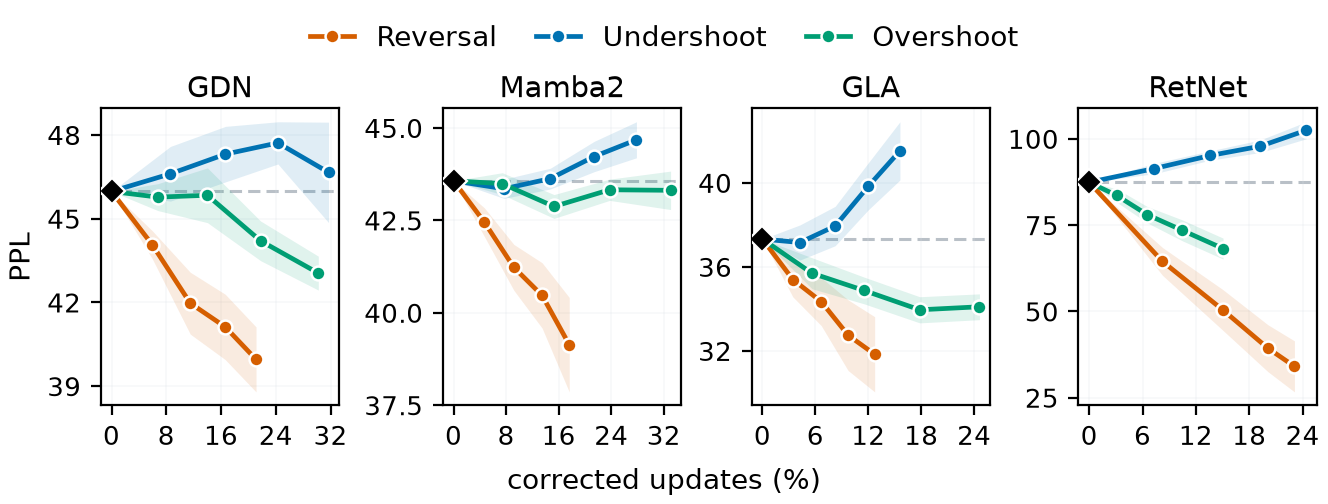}
  \caption{Left: three ways rounding distorts an update $d$, and the correction of
BDR. Right: perplexity after removing the rounding errors of each case.}
  \label{fig:correction}
  \vspace{-0.4cm}
\end{figure}

\paragraph{Bounded directional rounding.}
Since reversals are the most harmful, we correct them by moving the rounded
element $\mathcal Q_t(s+d)$ one grid step in the direction of $d$, as marked by
the star in Fig.~\ref{fig:correction} (left).
The effective update then follows $d$.
The corrected element, however, lies farther from $s+d$, so the rounding error
grows and the reversal becomes an overshoot
(Lemma~\ref{lem:correction}\,(ii)--(iii)).
Since the move spans a full grid step, the corrected update exceeds
$\Delta_t/2$ even when $d$ is tiny, which amplifies a small $d$ many times
(Lemma~\ref{lem:correction}\,(ii)).
To bound this amplification, we correct a reversal only when
$|d|>\epsilon_B\Delta_t$.
This gives BDR, which rounds $s+d$ as
\begin{equation}
\mathcal Q_t^{\mathrm{BDR}}(s+d)=
\begin{cases}
\mathcal Q_t(s+d)+\operatorname{sign}(d)\,\Delta_t,
& \text{if } \widetilde d\,d<0 \text{ and } |d|>\epsilon_B\Delta_t,\\
\mathcal Q_t(s+d), & \text{otherwise}.
\end{cases}
\label{eq:bdr}
\end{equation}
The corrected update is then less than $|d|/\epsilon_B$
(Lemma~\ref{lem:correction}\,(iv)).
A smaller $\epsilon_B$ thus corrects more reversals but allows larger
amplification, and we examine this trade-off in App.~\ref{app:ablation}.
Although BDR enlarges the rounding error (Lemma~1~(iii)), it improves the model (App.~A.5), suggesting that the direction of an update matters more than its rounding error at a single step.
Since BDR only changes the rounding rule, it requires no additional storage.
Fig.~\ref{fig:overview} (right) illustrates the procedure of BDR, where the
direction corrector applies Eq.~\ref{eq:bdr} to the reversed elements with
$|d|>\epsilon_B\Delta_t$.

%% file: sec/5_eval.tex
\vspace{-0.2cm}
\section{Experiments}
\label{sec:exp}
\vspace{-0.1cm}

\begin{table*}[t]
\centering
\footnotesize
\setlength{\tabcolsep}{4pt}
\renewcommand{\arraystretch}{0.85}
\setlength{\aboverulesep}{1.5pt}
\setlength{\belowrulesep}{1.5pt}
\begin{tabular}{llccccc|cccc}
\toprule
& & & \multicolumn{4}{c}{PG-19 PPL $\downarrow$} & \multicolumn{4}{c}{WikiText-103 PPL $\downarrow$} \\
\cmidrule(lr){4-7}\cmidrule(lr){8-11}
Model & Method & Bits & 1K & 2K & 4K & 8K & 1K & 2K & 4K & 8K \\
\midrule
\multirow{4}{*}{GDN}
& FP32     & 32.0 & 84.78 & 76.12 & 70.64 & 66.84 & 30.84 & 26.38 & 24.85 & 24.19 \\
& INT4     & 4.50 & 96.21 & 101.08 & 141.76 & 220.36 & 36.10 & 36.49 & 48.15 & 79.76 \\
& INT4+Had & 4.50 & 95.56 & 96.98 & 112.52 & 137.94 & 35.23 & 34.69 & 40.92 & 53.71 \\
\rowcolor{highlight}
& Ours     & 5.25 & 85.91 & 77.90 & 74.11 & 72.28 & 31.38 & 27.23 & 26.27 & 26.85 \\
\midrule
\multirow{4}{*}{GLA}
& FP32     & 32.0 & 64.71 & 60.89 & 57.96 & 55.85 & 25.06 & 21.66 & 20.69 & 21.57 \\
& INT4     & 4.50 & 93.64 & 104.05 & 125.15 & 158.79 & 30.02 & 32.53 & 41.68 & 66.55 \\
& INT4+Had & 4.50 & 88.28 & 89.00 & 94.07 & 99.13 & 31.22 & 31.21 & 34.18 & 41.69 \\
\rowcolor{highlight}
& Ours     & 5.31 & 65.82 & 62.51 & 60.89 & 62.24 & 25.47 & 22.36 & 22.19 & 24.86 \\
\midrule
\multirow{4}{*}{Mamba2}
& FP32     & 32.0 & 77.88 & 70.40 & 66.40 & 63.99 & 29.43 & 26.20 & 25.32 & 25.19 \\
& INT4     & 4.50 & 86.79 & 81.90 & 83.45 & 87.96 & 32.96 & 31.65 & 33.89 & 37.43 \\
& INT4+Had & 4.50 & 88.59 & 84.87 & 86.01 & 88.21 & 32.22 & 30.54 & 32.36 & 35.74 \\
\rowcolor{highlight}
& Ours     & 5.64 & 80.30 & 73.37 & 70.68 & 69.72 & 30.30 & 27.60 & 27.66 & 28.78 \\
\midrule
\multirow{4}{*}{RetNet}
& FP32     & 32.0 & 66.45 & 62.46 & 61.07 & 60.02 & 25.66 & 22.38 & 22.51 & 24.30 \\
& INT4     & 4.50 & 190.65 & 255.75 & 533.85 & 3170.42 & 59.22 & 75.61 & 179.61 & 1193.21 \\
& INT4+Had & 4.50 & 147.93 & 203.44 & 460.38 & 2749.92 & 46.25 & 60.32 & 142.60 & 918.77 \\
\rowcolor{highlight}
& Ours     & 5.25 & 69.66 & 67.88 & 70.22 & 75.96 & 26.68 & 23.76 & 24.97 & 28.80 \\
\bottomrule
\end{tabular}
\caption{Language modeling quality of 1.3B models on PG-19 and WikiText-103.
INT4+Had applies a Hadamard rotation before INT4 quantization. Bits is the
average storage per state element over decoding. The configurations of our method are listed in Table~\ref{tab:config}.}
\vspace{-0.2cm}
\label{tab:wiki_ppl}
\end{table*}

\paragraph{Models.}
We evaluate four recurrent models at the 1.3B scale: Gated DeltaNet
~\citep{yang2025gated}, GLA~\citep{yang2024gated},
Mamba2~\citep{dao2024transformers}, and RetNet~\citep{sun2023retentive}.
For GLA and RetNet, we use the public checkpoints released by the
flash-linear-attention (FLA) library\footnote{\url{https://huggingface.co/fla-hub}}.
Since FLA provides no 1.3B checkpoints for GDN and Mamba2, we pretrain them
ourselves on 100B tokens of FineWeb-Edu~\citep{penedo2024fineweb} with a sequence
length of 4,096.
Training uses AdamW~\citep{loshchilov2017decoupled} with a peak learning rate of
$4\times10^{-4}$, cosine decay, a weight decay of 0.1, and a batch size of 0.5M
tokens.
All state quantization methods compared below are applied directly to these
checkpoints without fine-tuning.
The hyperparameters of our method for each model are listed in
Table~\ref{tab:config}, and ablation studies are provided in
App.~\ref{app:ablation}.

\paragraph{Language modeling and commonsense reasoning.}
We measure word-level perplexity on PG-19~\citep{Rae2020Compressive} and
WikiText-103~\citep{merity2016pointer} at context lengths from 1K to 8K tokens
(Table~\ref{tab:wiki_ppl}).
We further evaluate LAMBADA~\citep{paperno2016lambada} and commonsense reasoning
on SciQ~\citep{welbl2017crowdsourcing}, COPA~\citep{roemmele2011choice}, and
HellaSwag~\citep{zellers2019hellaswag} with the LM Evaluation
Harness~\citep{gao2021framework} (Table~\ref{tab:zeroshot}).
We compare our method with full-precision FP32 states and two 4-bit baselines.
INT4 stores the state as 4-bit integers using the quantizer of
App.~\ref{app:quant} at every step, and INT4+Had applies a Hadamard rotation
before INT4 quantization to suppress outliers~\citep{ashkboos2024quarot}.
First, the degradation of INT4 grows with context length and follows the
coherent gain in Sec.~\ref{sec:2-2}, being mildest for Mamba2 and most severe
for RetNet.
Likewise, INT4 clearly degrades LAMBADA, which relies on a broad context, but
changes the commonsense tasks with short inputs only slightly.
Second, INT4+Had improves INT4 at long contexts on GDN, GLA, and RetNet, but
still degrades sharply as the context grows.
Reducing the per-step error thus cannot prevent the accumulation analyzed in
Sec.~\ref{sec:error}.
Third, our method stays far closer to FP32 than both baselines on all four
models, and its advantage grows with context length.
On LAMBADA, it brings perplexity close to FP32 and keeps accuracy within 0.5
points of FP32 on GDN, GLA, and Mamba2.

\begin{table*}[t]
\centering
\small
\setlength{\tabcolsep}{5pt}
\renewcommand{\arraystretch}{0.8}
\setlength{\aboverulesep}{1.5pt}
\setlength{\belowrulesep}{1.5pt}
\begin{tabular}{ll|ccccc}
\toprule
Model & Method & LMB-PPL $\downarrow$ & LMB-Acc $\uparrow$ & SciQ $\uparrow$ & COPA $\uparrow$ & HS $\uparrow$ \\
\midrule
\multirow{3}{*}{GDN}
 & FP32 & 12.53 & 46.38 & 90.9 & 73.0 & 55.79 \\
 & INT4   & 17.94 & 39.90 & 90.3 & 74.0 & 55.54 \\
\rowcolor{highlight}
 & Ours & 12.85 & 46.01 & 91.2 & 73.0 & 55.84 \\
\midrule
\multirow{3}{*}{GLA}
 & FP32 & 15.56 & 46.46 & 80.7 & 70.0 & 48.86 \\
 & INT4   & 22.48 & 38.02 & 79.1 & 69.0 & 48.51 \\
\rowcolor{highlight}
 & Ours & 15.75 & 46.05 & 80.7 & 70.0 & 48.86 \\
\midrule
\multirow{3}{*}{Mamba2}
 & FP32 & 10.82 & 50.18 & 88.9 & 80.0 & 59.36 \\
 & INT4   & 12.13 & 47.76 & 89.4 & 81.0 & 59.05 \\
\rowcolor{highlight}
 & Ours & 10.89 & 49.85 & 89.5 & 79.0 & 59.27 \\
\midrule
\multirow{3}{*}{RetNet}
 & FP32 & 17.00 & 43.08 & 83.3 & 72.0 & 48.14 \\
 & INT4   & 19.00 & 42.67 & 81.5 & 73.0 & 46.64 \\
\rowcolor{highlight}
 & Ours & 17.84 & 42.42 & 82.5 & 72.0 & 48.07 \\
\bottomrule
\end{tabular}
\caption{Zero-shot performance of 1.3B models on language modeling and commonsense
reasoning benchmarks. LMB: LAMBADA; HS: HellaSwag, reported as
length-normalized accuracy.}
\vspace{-0.cm}
\label{tab:zeroshot}
\end{table*}

\paragraph{Recall-intensive tasks.}
We further evaluate the recall-intensive tasks collected by
\citet{arora2024just}, as shown in Table~\ref{tab:recall}.
These tasks require the model to retrieve specific information from the context,
and thus directly test the key--value associations stored in the state.
They include information extraction on FDA~\citep{arora2023language} and
SWDE~\citep{lockard2019openceres}, and question answering on
NQ~\citep{kwiatkowski2019natural}, SQuADv2~\citep{rajpurkar2018know},
TriviaQA~\citep{joshi2017triviaqa}, and DROP~\citep{dua2019drop}.
First, INT4 substantially degrades recall even at contexts of only 512 and 1K
tokens.
The drop is largest on FDA, which requires copying exact values from the context.
Second, our method narrows the gap to FP32 on all four models.
On GDN and GLA, its average accuracy is within about one point of FP32.
On RetNet, where INT4 loses the most, it recovers most of the lost accuracy.
Hence, preserving state updates also preserves the ability to recall information
from the context.

\begin{table*}[t]
\centering
\small
\setlength{\tabcolsep}{5pt}
\renewcommand{\arraystretch}{0.85}
\setlength{\aboverulesep}{1.5pt}
\setlength{\belowrulesep}{1.5pt}
\begin{tabular}{llccccccccc|c}
\toprule
Model & Method & F-512 & F-1K & S-512 & S-1K & N-512 & N-1K & SQ & TQA & DR & Avg \\
\midrule
\multirow{3}{*}{GDN}
 & FP32 & 51.2 & 41.1 & 42.9 & 36.7 & 31.1 & 31.1 & 41.0 & 57.7 & 21.7 & 39.4 \\
 & INT4   & 30.1 & 25.4 & 33.7 & 25.5 & 26.5 & 24.4 & 39.0 & 56.7 & 16.3 & 30.8 \\
\rowcolor{highlight}
 & Ours & 52.5 & 39.1 & 41.2 & 35.7 & 28.2 & 29.4 & 41.7 & 58.0 & 20.7 & 38.5 \\
\midrule
\multirow{3}{*}{GLA}
 & FP32 & 64.9 & 46.8 & 44.2 & 31.3 & 34.0 & 30.3 & 40.3 & 53.7 & 20.7 & 40.7 \\
 & INT4   & 20.4 & 13.7 & 36.1 & 28.2 & 27.7 & 25.2 & 33.7 & 47.3 & 17.0 & 27.7 \\
\rowcolor{highlight}
 & Ours & 59.9 & 41.1 & 41.2 & 31.6 & 34.5 & 30.7 & 41.0 & 52.7 & 24.7 & 39.7 \\
\midrule
\multirow{3}{*}{Mamba2}
 & FP32 & 48.5 & 34.1 & 41.5 & 29.9 & 35.7 & 34.0 & 38.7 & 58.7 & 20.0 & 37.9 \\
 & INT4   & 28.1 & 15.4 & 31.0 & 19.0 & 29.8 & 23.1 & 39.3 & 55.7 & 20.0 & 29.0 \\
\rowcolor{highlight}
 & Ours & 36.5 & 24.1 & 35.7 & 21.1 & 34.9 & 27.7 & 40.0 & 57.3 & 20.7 & 33.1 \\
\midrule
\multirow{3}{*}{RetNet}
 & FP32 & 68.2 & 39.8 & 48.0 & 32.3 & 29.4 & 28.6 & 44.7 & 47.3 & 23.0 & 40.1 \\
 & INT4   & 13.7 & 4.7  & 31.0 & 17.3 & 17.2 & 14.7 & 34.7 & 43.3 & 16.0 & 21.4 \\
\rowcolor{highlight}
 & Ours & 59.9 & 32.8 & 45.6 & 25.5 & 26.5 & 23.9 & 42.0 & 47.3 & 20.7 & 36.0 \\
\bottomrule
\end{tabular}
\caption{Accuracy $\uparrow$ on recall-intensive tasks~\citep{arora2024just} with 1.3B models. F: FDA, S: SWDE, N: NQ, SQ: SQuADv2, TQA: TriviaQA, DR: DROP; 512 and 1K denote the context length.}
\label{tab:recall}
\vspace{-0.4cm}
\end{table*}

\paragraph{Efficiency.}
We implement each decoding step of our method with fused Triton kernels.
Fig.~\ref{fig:efficiency} compares our method with FP32 states on an RTX 3060 GPU
at batch sizes of 16, 32, and 48.
The left panel reports the memory occupied by the recurrent states of all layers,
excluding weights and activations.
The right panel reports the end-to-end decoding throughput of the whole model,
measured under the same setup for both methods.
First, our method reduces state memory by $5.3\times$ on GDN, GLA, and RetNet,
and by $4.7\times$ on Mamba2.
The reduction is smaller on Mamba2 because its state per head is smaller, so the
bypass takes a larger share of the storage.
Second, at a batch size of 48, our method speeds up decoding by
1.3--1.9$\times$, although LRB and BDR add extra computation.
This is because decoding reads and writes the entire state at every step, so
storing the state in low precision reduces memory traffic.
Moreover, the speedup grows with batch size, as the state becomes a larger
bottleneck in larger batches.

\begin{figure}
    \centering
    \includegraphics[width=0.95\linewidth]{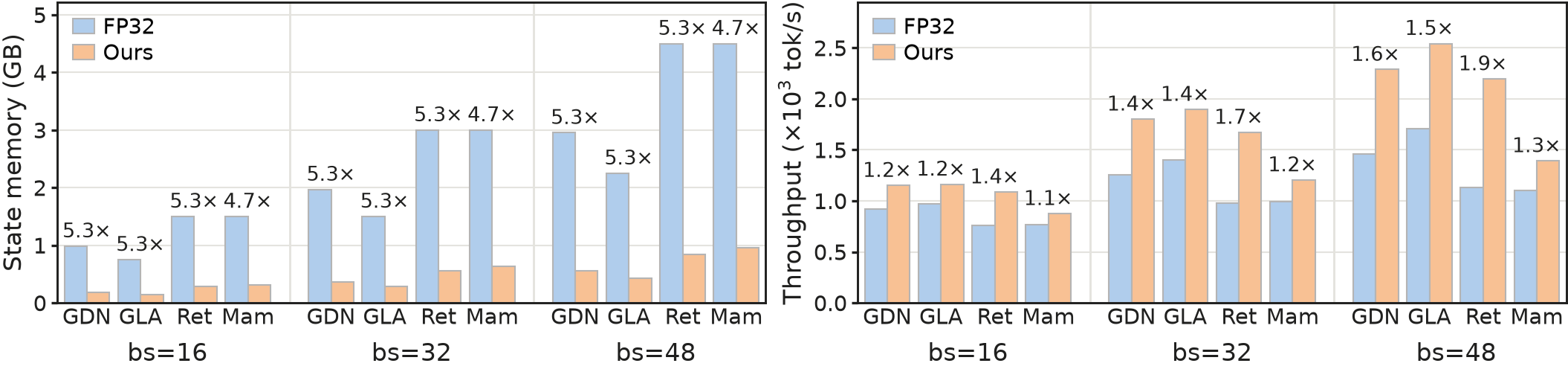}
    \vspace{-0.2cm}
\caption{State memory and end-to-end decoding throughput of our
method and FP32 states.}
\vspace{-0.2cm}
\label{fig:efficiency}
\end{figure}

%% file: sec/6_related.tex
\section{Related Work}
\label{sec:related}

\paragraph{LLM quantization.}
Post-training quantization reduces the inference cost of large language models
(LLMs) through low-precision representations.
For model weights, GPTQ~\citep{frantar2022gptq} uses approximate second-order
information and AWQ~\citep{lin2024awq} uses activation statistics to reduce
quantization error.
Moreover, SmoothQuant~\citep{xiao2023smoothquant}, OmniQuant~\citep{shao2024omniquant},
and QuaRot~\citep{ashkboos2024quarot} quantize activations by suppressing
outliers, while KIVI~\citep{liu2024kivi} and KVQuant~\citep{hooper2024kvquant}
extend quantization to the KV cache.
All these methods quantize each tensor once and thus aim to reduce the error of
that single quantization.
In contrast, a recurrent state is requantized at every step, so its errors
accumulate through the recurrence as analyzed in Sec.~\ref{sec:2-2}.

\vspace{-0.1cm}
\paragraph{Quantization error compensation.}
Low-precision training faces a similar problem, as small weight updates can be
lost to rounding.
Mixed-precision training therefore keeps full-precision master
weights~\citep{micikevicius2017mixed}, and \citet{zamirai2020revisiting} apply
stochastic rounding and Kahan summation to BFloat16 training.
Similarly, error feedback~\citep{karimireddy2019error} stores the residual of
gradient compression and adds it to the next gradient before compression.
Master weights, Kahan summation, and error feedback all keep a buffer as large as
the tensor they protect, which is acceptable because their main savings lie in
computation, activation memory, or communication.
In contrast, our goal is to reduce state storage.
Such a buffer would take as much memory as the full-precision state.
LRB therefore keeps a low-rank buffer by exploiting the outer-product structure of
state updates.
Moreover, unlike weights stored on a fixed floating-point grid, a recurrent state
is rounded onto a grid that is refitted each time, so rounding can also reverse an
update, which BDR corrects.

\vspace{-0.1cm}
\paragraph{Recurrent state quantization.}
Quantization of recurrent models has mainly targeted Mamba models.
Quamba~\citep{quamba}, MambaQuant~\citep{xu2025mambaquant}, and
Quamba2~\citep{quamba2} focus on weights and activations, and Quamba2 also
quantizes the cached states.
Q-Mamba~\citep{tianqi-etal-2025-q} quantizes the state cache of Mamba2 to 4 bits with separate scales along the state and channel dimensions, and calibrates the quantized model by block-wise reconstruction on calibration data.
DAMP~\citep{zhang2026damp} quantizes the states of GDN and Kimi Delta Attention
by keeping the riskiest key channels in FP16 and the rest in INT8, with the risk
estimated offline from quantization-error energy and decay-based persistence.
Our method differs from these works in two ways.
First, whereas these methods reduce the error of each state quantization, our
method targets the coherence of errors across steps by preserving the updates
that rounding would lose or reverse.
Second, it requires no gradient-based fine-tuning or calibration statistics, and applies to
diverse recurrent models, including GDN, GLA, Mamba2, and RetNet.

%% file: sec/7_discussion.tex
\section{Conclusion}
\label{sec:conclusion}
\vspace{-0.2cm}

We study low-bit storage of recurrent states from the perspective of the updates
written into them.
We show that per-step quantization errors accumulate coherently through the
recurrence, and that this coherence arises mainly from updates lost in rounding.
We further find that rounding can also reverse updates, which is the most harmful
way it distorts an update.
Guided by these findings, LRB preserves the magnitude of small updates before
rounding, and BDR preserves their direction at rounding while bounding the
resulting overshoot.
Across GDN, GLA, Mamba2, and RetNet, our method keeps model quality much closer to
FP32 than INT4, while reducing state memory by 4.7--5.3$\times$ and speeding up
decoding by up to 1.9$\times$.

\paragraph{Discussion and Limitations.}
Our evaluation covers four recurrent architectures, and extending it to more
models, including hybrid models that combine recurrent layers with softmax
attention~\citep{chandiramani2026nemotron}, is a natural next step.
Moreover, we quantize only the recurrent states and keep weights and activations
in their original precision, so combining our method with weight and activation
quantization may further reduce inference cost.

%% file: sec/appendix_proof.tex
\newpage
\section{Appendix}
\label{sec:appendix}

\subsection{State Quantization Details}
\label{app:quant}

\paragraph{Layout and grouping.}
For every model, we write the per-head state as
$\mathbf S_t\in\mathbb R^{d_k\times d_v}$, with rows indexed by key dimensions
and columns by value dimensions.
The state is quantized in groups of 128 contiguous elements along the key
dimension, so each group lies within a single column.
All groups use the same bit width and are quantized independently.

\paragraph{Quantizer.}
Let $x_1,\dots,x_{128}$ be the elements of one group to be stored at step $t$.
The quantizer fits an $n$-bit uniform grid to their range and rounds each element
to the nearest grid point:
\begin{equation}
\Delta_t=\frac{\max_i x_i-\min_i x_i}{2^n-1},
\qquad
z_t=\Big\lfloor-\frac{\min_i x_i}{\Delta_t}\Big\rceil,
\end{equation}
\begin{equation}
q_i=\operatorname{clamp}\!\Big(\Big\lfloor\frac{x_i}{\Delta_t}\Big\rceil+z_t,\ 0,\ 2^n-1\Big),
\qquad
\mathcal Q_t(x_i)=\Delta_t\,(q_i-z_t),
\end{equation}
where $\lfloor\cdot\rceil$ denotes rounding to the nearest integer and
$\operatorname{clamp}(x,a,c)=\min(\max(x,a),c)$.
Since $z_t$ is an integer, every grid point is an integer multiple of $\Delta_t$,
so zero is represented exactly.
Because the grid covers the range of the group up to the rounding of $z_t$,
nearest rounding changes each element by at most $\Delta_t/2$.
We use $n=4$ and store the integer codes $q_i$, together with $\Delta_t$ and
$z_t$ in FP32 once per group, giving an average of
$(4\cdot 128+2\cdot 32)/128=4.5$ bits per state element.

\paragraph{Refitting and computation.}
The INT4 baseline quantizes the state at every step, whereas our method quantizes
it only at merges.
In both cases, $\Delta_t$ and $z_t$ are refitted to the state being stored,
namely $\mathbf A_t\widehat{\mathbf S}_{t-1}+\mathbf B_t$ in
Eq.~\ref{eq:quant} or
$\mathbf A_{t_0:t}\widehat{\mathbf S}_{t_0}+\mathbf D_t$ in
Eq.~\ref{eq:merge}, so the grid of each group changes over time.
Between merges, our method keeps the stored codes unchanged and applies the
accumulated decay when the state is read.
This accumulated decay adds $d_k$ FP32 scalars per head for GLA and one FP32 scalar
per head for the other models.
All computation is performed in FP32: the stored state is dequantized, decayed,
and updated before being quantized again.
The output at each step is read from the updated state before it is quantized.

\paragraph{Evaluation protocol.}
In all evaluations, every method runs each model in its recurrent form from an empty state and processes one token at a time, including the context tokens. For the quantized methods, the state is kept in low-precision storage throughout, exactly as in decoding, so that a context of $T$ tokens emulates $T$ decoding steps with a quantized state.

\subsection{Statistics of Lost Updates}
\label{app:lost_stats}

Table~\ref{tab:lost_stats} supports the claim in Sec.~\ref{sec:2-3} that a single
update is typically much smaller than the grid spacing and is often lost in
rounding.
We use the same INT4 trajectories on WikiText-103 \citep{merity2016pointer} as in Fig.~1, where the updates $\mathbf{B}_t$ are those computed by the INT4 model itself.
For each nonzero element $b$ of $\mathbf B_t$, we compute its ratio $|b|/\Delta_t$
to the grid spacing of its group, and report the quartiles $Q_{25}$, $Q_{50}$,
and $Q_{75}$ of this ratio.
We also report the fraction of these elements whose update is lost, i.e., whose
rounding result stays unchanged.

\begin{table}[h]
\centering
\small
\begin{tabular}{lcccc}
\toprule
Model & $Q_{25}$ & $Q_{50}$ & $Q_{75}$ & Lost fraction \\
\midrule
GDN    & 0.060  & 0.207 & 0.582 & 68.0\% \\
GLA    & 0.021  & 0.072 & 0.209 & 83.7\% \\
Mamba2 & 0.052  & 0.188 & 0.543 & 72.4\% \\
RetNet & 0.005 & 0.017 & 0.046 & 96.4\% \\
\bottomrule
\end{tabular}
\caption{Quartiles of $|b|/\Delta_t$ and fraction of lost updates under INT4 at
$T=4$K on WikiText-103.}
\label{tab:lost_stats}
\end{table}

On all four models, the median update is at most about 20\% of the grid
spacing, and 68--96\% of the updates are lost at each step.
RetNet has the smallest updates relative to the grid spacing and loses almost all
of them, consistent with its largest coherent gain in Sec.~\ref{sec:2-2}.

\subsection{Properties of BDR}
\label{app:proofs}

Consider a single element at a rounding step.
As in Sec.~\ref{sec:bdr}, let $s$ be its element in the decayed state, $d$ the
update written to it, and $\Delta_t$ the grid spacing of its group.
Nearest rounding stores $\mathcal Q_t(s+d)$, which gives the effective update
$\widetilde d=\mathcal Q_t(s+d)-s$.
When BDR corrects the element, it stores
$\mathcal Q_t(s+d)+\operatorname{sign}(d)\,\Delta_t$ (Eq.~\ref{eq:bdr}), which
gives the corrected update
$\widetilde d'=\widetilde d+\operatorname{sign}(d)\,\Delta_t$.
We assume that this grid point lies within the range of the grid; otherwise,
BDR leaves the element unchanged.

\begin{lemma}
\label{lem:correction}
Suppose that the update is reversed, i.e., $\widetilde d\,d<0$. Then:
\begin{enumerate}
  \item[(i)] $\mathcal Q_t(s+d)=\mathcal Q_t(s)$, so the update is lost in
  rounding; moreover, $|d|+|\widetilde d|\le\Delta_t/2$, and in particular
  $|d|<\Delta_t/2$;
  \item[(ii)] $\widetilde d'$ has the sign of $d$, and
  $|d|+\Delta_t/2\le|\widetilde d'|<\Delta_t$, so the correction turns the
  reversal into an overshoot;
  \item[(iii)] $|\widetilde d'-d|\ge\Delta_t/2\ge|\widetilde d-d|$, so the
  correction enlarges the rounding error;
  \item[(iv)] if $|d|>\epsilon_B\Delta_t$, then $|\widetilde d'|<|d|/\epsilon_B$.
\end{enumerate}
\end{lemma}

\begin{proof}
By symmetry, assume $d>0$, so that $\widetilde d<0$ and
$\widetilde d'=\Delta_t-|\widetilde d|$.
Nearest rounding changes each element by at most $\Delta_t/2$
(App.~\ref{app:quant}), so $|\mathcal Q_t(s+d)-(s+d)|\le\Delta_t/2$.

\emph{(i)} Since $\widetilde d<0$, we have $\mathcal Q_t(s+d)<s<s+d$.
Nearest rounding is monotone, and $\mathcal Q_t(s+d)$ is itself a grid point, so
$\mathcal Q_t(s+d)=\mathcal Q_t\bigl(\mathcal Q_t(s+d)\bigr)\le\mathcal Q_t(s)
\le\mathcal Q_t(s+d)$, which gives $\mathcal Q_t(s)=\mathcal Q_t(s+d)$.
Moreover, the distance from $\mathcal Q_t(s+d)$ to $s+d$ is
$(s+d)-\mathcal Q_t(s+d)=|d|+|\widetilde d|$, which is at most $\Delta_t/2$.
Since $\widetilde d\neq0$, this also gives $|d|<\Delta_t/2$.

\emph{(ii)} By (i), $|\widetilde d|\le\Delta_t/2-|d|$, so
$\widetilde d'=\Delta_t-|\widetilde d|\ge|d|+\Delta_t/2>|d|>0$.
Since $|\widetilde d|>0$, we also have $\widetilde d'<\Delta_t$.

\emph{(iii)} Before the correction, $\widetilde d$ and $d$ have opposite signs, so
$|\widetilde d-d|=|d|+|\widetilde d|\le\Delta_t/2$ by (i).
After the correction, $\widetilde d'>d>0$ by (ii), so
$|\widetilde d'-d|=\Delta_t-|\widetilde d|-|d|\ge\Delta_t/2$, again by (i).

\emph{(iv)} By (ii), $|\widetilde d'|<\Delta_t$, and the condition
$|d|>\epsilon_B\Delta_t$ gives $\Delta_t<|d|/\epsilon_B$.
\end{proof}

\subsection{Implementation}
\label{app:impl}

\paragraph{LRB.} The factors of the bypass are stored in FP16.
We compute $\mathcal T_k$ with a randomized SVD~\citep{halko2011finding} using
$k+8$ random projections and two power iterations.
Each decoding step is implemented with fused Triton kernels, and $\mathbf L_t$
is materialized on chip only at merges.
The reported bits per state element include the storage of the bypass.

\subsection{Hyperparameters and Ablations}
\label{app:ablation}

Our method has four hyperparameters: the rank $k$ that the bypass keeps after each
merge, the merge interval $m$, the threshold $\epsilon_L$ below which accumulated
updates stay in the bypass, and the threshold $\epsilon_B$ above which BDR
corrects a reversal.
Table~\ref{tab:config} lists the configuration of each model, which we use in all
experiments of Sec.~\ref{sec:exp}.
We then examine the contributions of LRB and BDR and the effect of each
hyperparameter.

\begin{table}[h]
\centering
\small
\begin{tabular}{lcccccc}
\toprule
Model & $k$ & $m$ & $\epsilon_L$ & $\epsilon_B$ & Avg.\ bits & Peak bits \\
\midrule
GDN    & 1 & 15 & 0.50 & 0.30 & 5.25 & 6.00 \\
GLA    & 1 & 15 & 0.50 & 0.30 & 5.31 & 6.07 \\
Mamba2 & 1 & 5  & 0.10 & 0.25 & 5.64 & 6.78 \\
RetNet & 1 & 15 & 0.70 & 0.30 & 5.25 & 6.00 \\
\bottomrule
\end{tabular}
\caption{Configuration of our method for each model. Avg.\ bits and Peak bits
are the average and peak storage per state element during decoding.}
\label{tab:config}
\end{table}

\begin{table*}[h]
\centering
\footnotesize
\setlength{\tabcolsep}{4pt}
\renewcommand{\arraystretch}{0.9}
\setlength{\aboverulesep}{1.5pt}
\setlength{\belowrulesep}{1.5pt}
\begin{tabular*}{\textwidth}{@{\extracolsep{\fill}}lrrrrrrrr@{}}
\toprule
& \multicolumn{4}{c}{2K (PPL $\downarrow$)} & \multicolumn{4}{c}{8K (PPL $\downarrow$)} \\
\cmidrule(lr){2-5}\cmidrule(lr){6-9}
Method & GDN & GLA & Mamba2 & RetNet & GDN & GLA & Mamba2 & RetNet \\
\midrule
FP32 & 26.38 & 21.66 & 26.20 & 22.38 & 24.19 & 21.57 & 25.19 & 24.30 \\
INT4 & 36.49 & 32.53 & 31.65 & 75.61 & 79.76 & 66.55 & 37.43 & 1193.21 \\
BDR & 34.22 & 32.49 & 30.85 & 57.92 & 63.15 & 75.61 & 35.10 & 707.65 \\
LRB & 27.37 & 22.57 & 28.00 & 23.88 & 27.64 & 25.91 & 29.90 & 31.18 \\
LRB+BDR & \textbf{27.23} & \textbf{22.36} & \textbf{27.60} & \textbf{23.76} & \textbf{26.85} & \textbf{24.86} & \textbf{28.78} & \textbf{28.80} \\
\bottomrule
\end{tabular*}
\caption{Component ablation on WikiText-103 with the configurations in
Table~\ref{tab:config}.}
\label{tab:ablation_components}
\end{table*}

\paragraph{Components.}
Table~\ref{tab:ablation_components} isolates the contributions of LRB and BDR.
First, without LRB, the state is rounded at every step, and BDR alone corrects
the reversals after each step.
It reduces the perplexity of INT4 in all settings except GLA at 8K, especially
on RetNet, the model most affected by INT4 quantization.
The exception on GLA is likely because $\epsilon_B$ is chosen for LRB+BDR
rather than for BDR alone.
Second, LRB alone provides a larger improvement, bringing perplexity close to
FP32 at 2K and preventing the sharp increase of INT4 at 8K.
Third, combining the two gives the best results on all four models at both
context lengths, with the largest gain on RetNet at 8K.
The two components thus complement each other: LRB preserves the magnitude of
small updates, and BDR preserves the direction of those still lost in rounding.

\begin{table}[h]
\centering
\footnotesize
\setlength{\tabcolsep}{4pt}
\renewcommand{\arraystretch}{0.95}
\setlength{\aboverulesep}{1.5pt}
\setlength{\belowrulesep}{1.5pt}
\begin{tabular}{lcccccc}
\toprule
& & & \multicolumn{2}{c}{PPL $\downarrow$} & \multicolumn{2}{c}{Bits} \\
\cmidrule(lr){4-5}\cmidrule(lr){6-7}
Model & $k$ & $m$ & 2K & 8K & Avg. & Peak \\
\midrule
\multirow{5}{*}{GDN}
& 0 & 16 & 27.34 & 27.27 & 5.20 & \multirow{5}{*}{6.00} \\
& 1 & 15 & \textbf{27.23} & \textbf{26.85} & 5.25 & \\
& 2 & 14 & 27.33 & 26.97 & 5.30 & \\
& 4 & 12 & 27.33 & 27.29 & 5.39 & \\
& 8 & 8  & 27.52 & 27.98 & 5.58 & \\
\midrule
\multirow{5}{*}{GLA}
& 0 & 16 & \textbf{22.35} & 25.63 & 5.26 & \multirow{5}{*}{6.07} \\
& 1 & 15 & 22.36 & \textbf{24.86} & 5.31 & \\
& 2 & 14 & 22.48 & 25.30 & 5.36 & \\
& 4 & 12 & 22.53 & 26.13 & 5.45 & \\
& 8 & 8  & 23.11 & 28.58 & 5.64 & \\
\midrule
\multirow{5}{*}{RetNet}
& 0 & 16 & 24.26 & 77.86 & 5.20 & \multirow{5}{*}{6.00} \\
& 1 & 15 & \textbf{23.76} & \textbf{28.80} & 5.25 & \\
& 2 & 14 & 24.09 & 28.96 & 5.30 & \\
& 4 & 12 & 24.28 & 29.62 & 5.39 & \\
& 8 & 8  & 25.53 & 32.74 & 5.58 & \\
\midrule
\multirow{4}{*}{Mamba2}
& 0 & 6 & 28.03 & 29.34 & 5.45 & \multirow{4}{*}{6.78} \\
& 1 & 5 & \textbf{27.60} & \textbf{28.78} & 5.64 & \\
& 2 & 4 & 27.70 & 29.10 & 5.83 & \\
& 3 & 3 & 27.91 & 29.47 & 6.02 & \\
\bottomrule
\end{tabular}
\caption{Effect of the bypass rank $k$ on WikiText-103 perplexity. The number
of bypass columns $k+m$ is fixed at 16 for GDN, GLA, and RetNet and at 6 for
Mamba2. Bits is the storage per state element.}
\label{tab:ablation_rank}
\end{table}

\begin{table}[h]
\centering
\footnotesize
\setlength{\tabcolsep}{4pt}
\renewcommand{\arraystretch}{0.95}
\setlength{\aboverulesep}{1.5pt}
\setlength{\belowrulesep}{1.5pt}
\begin{tabular}{lccccc}
\toprule
& & \multicolumn{2}{c}{PPL $\downarrow$} & \multicolumn{2}{c}{Bits} \\
\cmidrule(lr){3-4}\cmidrule(lr){5-6}
Model & $k+m$ & 2K & 8K & Avg. & Peak \\
\midrule
\multirow{4}{*}{GDN}
& 8  & 27.99 & 29.18 & 4.88 & 5.25 \\
& 12 & 27.49 & 27.71 & 5.06 & 5.63 \\
& 16 & 27.23 & 26.85 & 5.25 & 6.00 \\
& 24 & 26.94 & 26.07 & 5.63 & 6.76 \\
\midrule
\multirow{4}{*}{GLA}
& 8  & 23.16 & 29.23 & 4.94 & 5.31 \\
& 12 & 22.59 & 26.42 & 5.13 & 5.69 \\
& 16 & 22.36 & 24.86 & 5.31 & 6.07 \\
& 24 & 22.02 & 23.32 & 5.69 & 6.82 \\
\midrule
\multirow{4}{*}{RetNet}
& 8  & 25.69 & 44.86 & 4.88 & 5.25 \\
& 12 & 24.33 & 33.25 & 5.06 & 5.63 \\
& 16 & 23.76 & 28.80 & 5.25 & 6.00 \\
& 24 & 23.26 & 27.20 & 5.63 & 6.76 \\
\midrule
\multirow{3}{*}{Mamba2}
& 4 & 27.92 & 29.65 & 5.26 & 6.02 \\
& 6 & 27.60 & 28.78 & 5.64 & 6.78 \\
& 8 & 27.52 & 28.34 & 6.02 & 7.54 \\
\bottomrule
\end{tabular}
\caption{Effect of the number of bypass columns $k+m$ on WikiText-103
perplexity, with $k=1$. Bits is the storage per state element.}
\label{tab:ablation_size}
\end{table}

\paragraph{Bypass rank $k$.}
Table~\ref{tab:ablation_rank} varies the rank $k$ that the bypass keeps after each
merge.
To compare at the same peak storage, we fix the total number of bypass columns
$k+m$, so a larger $k$ shortens the merge interval $m$.
We fix $k+m=6$ for Mamba2, since its smaller state makes each bypass column
relatively more costly.
First, keeping a single rank instead of none improves perplexity at 8K on all four
models at almost no extra storage, most notably on RetNet.
Without this rank, small updates that remain below the threshold are merged into
the state and lost in rounding.
Second, increasing $k$ beyond 1 raises the average bits but worsens perplexity,
especially at 8K.
A larger $k$ shortens $m$, so the state is rounded more often, while the
additional ranks bring little benefit.

\paragraph{Bypass size $k+m$.}
Table~\ref{tab:ablation_size} varies the number of bypass columns $k+m$ with
$k=1$, which sets the merge interval $m$ and the storage of the bypass.
First, perplexity improves steadily as the bypass grows, but with diminishing
returns.
We choose $k+m=16$ as the default ($k+m=6$ for Mamba2), since larger bypasses
bring smaller gains.
Second, even the smallest bypass, which adds only 0.38--0.76 average bits to
INT4, already removes most of the degradation of INT4 at 8K (Table~\ref{tab:wiki_ppl}).
Most of the improvement therefore comes from preserving updates rather than from
the extra bits.

\begin{figure}[h]
    \centering
    \includegraphics[width=1.0\linewidth]{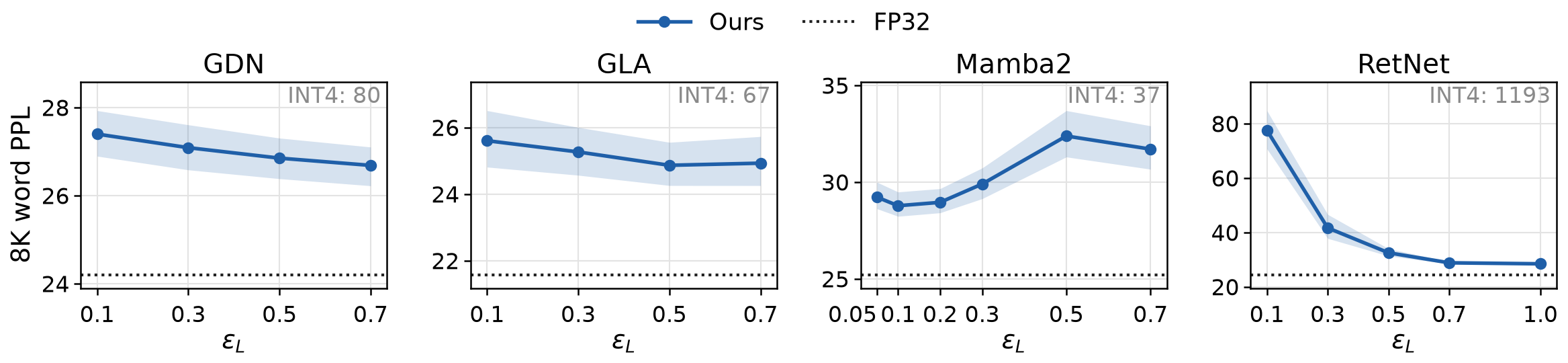}
    \caption{8K perplexity on WikiText-103 versus $\epsilon_L$, with the INT4
    perplexity noted in each panel.}
    \label{fig:epsl}
\end{figure}

\paragraph{Bypass threshold $\epsilon_L$.}
Fig.~\ref{fig:epsl} varies $\epsilon_L$, which decides which accumulated updates
stay in the bypass at each merge.
Across the tested range, our method stays far below INT4 on all four models.
The threshold balances two effects.
A small threshold merges small updates into the state, where rounding loses them,
whereas a large threshold keeps more updates than a rank-1 bypass can represent
accurately.
This balance differs across models.
RetNet, whose INT4 errors accumulate most coherently (Sec.~\ref{sec:2-2}),
benefits from larger thresholds, whereas Mamba2 favors smaller ones.

\begin{figure}[h]
    \centering
    \includegraphics[width=1.0\linewidth]{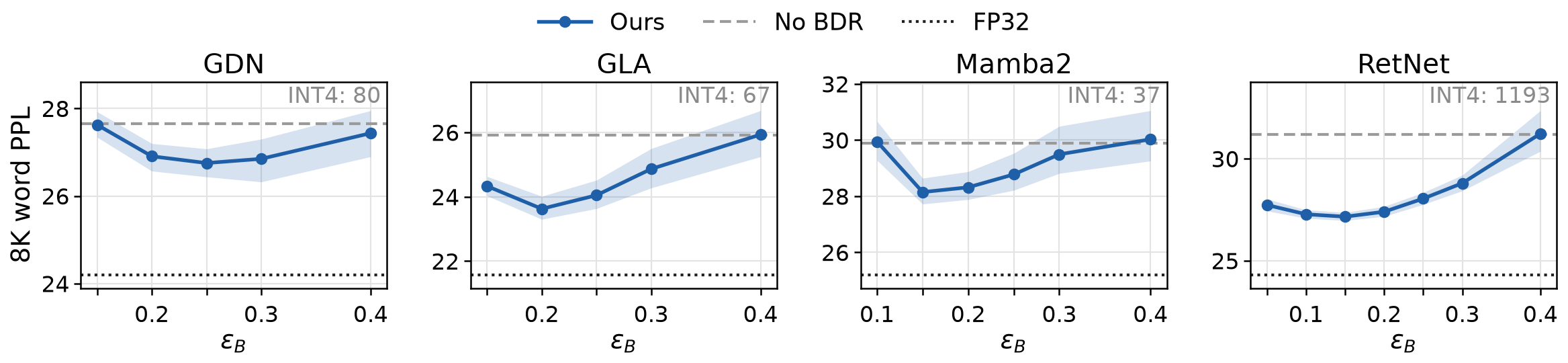}
    \caption{8K perplexity on WikiText-103 versus $\epsilon_B$, compared with LRB
    without BDR.}
    \label{fig:epsb}
\end{figure}

\paragraph{BDR threshold $\epsilon_B$.}
Fig.~\ref{fig:epsb} varies $\epsilon_B$, which bounds the overshoot introduced by
each correction (Sec.~\ref{sec:bdr}).
For $\epsilon_B$ between 0.2 and 0.3, BDR lowers perplexity below LRB alone on all
four models, most notably on RetNet.
As $\epsilon_B$ grows, BDR corrects fewer reversals, so its perplexity approaches
that of LRB alone.
Conversely, a very small $\epsilon_B$ also corrects the reversals of tiny updates,
whose large overshoots reduce the gain on GDN and GLA.

%% file: iclr2027_conference.bib
@article{frantar2022gptq,
  title={Gptq: Accurate post-training quantization for generative pre-trained transformers},
  author={Frantar, Elias and Ashkboos, Saleh and Hoefler, Torsten and Alistarh, Dan},
  journal={arXiv preprint arXiv:2210.17323},
  year={2022}
}

@article{lin2024awq,
  title={Awq: Activation-aware weight quantization for on-device llm compression and acceleration},
  author={Lin, Ji and Tang, Jiaming and Tang, Haotian and Yang, Shang and Chen, Wei-Ming and Wang, Wei-Chen and Xiao, Guangxuan and Dang, Xingyu and Gan, Chuang and Han, Song},
  journal={Proceedings of machine learning and systems},
  volume={6},
  pages={87--100},
  year={2024}
}

@inproceedings{xiao2023smoothquant,
  title={Smoothquant: Accurate and efficient post-training quantization for large language models},
  author={Xiao, Guangxuan and Lin, Ji and Seznec, Mickael and Wu, Hao and Demouth, Julien and Han, Song},
  booktitle={International conference on machine learning},
  pages={38087--38099},
  year={2023},
  organization={PMLR}
}

@inproceedings{shao2024omniquant,
  title={Omniquant: Omnidirectionally calibrated quantization for large language models},
  author={Shao, Wenqi and Chen, Mengzhao and Zhang, Zhaoyang and Xu, Peng and Zhao, Lirui and Li, Zhiqian and Zhang, Kaipeng and Peng, Gao and Qiao, Yu and Luo, Ping},
  booktitle={International Conference on Learning Representations},
  volume={2024},
  pages={45472--45496},
  year={2024}
}

@article{zhang2026damp,
  title={DAMP: Decay-Aware Mixed-Precision Recurrent-State Quantization},
  author={Zhang, Tao and Tan, Jianchao and Sun, Pingwei and Yu, Yanqi and Jiang, Zixu and Xie, Yuchen and Cai, Xunliang and Zeng, Ziqian},
  journal={arXiv preprint arXiv:2608.27513},
  year={2026}
}

@article{ashkboos2024quarot,
  title={Quarot: Outlier-free 4-bit inference in rotated llms},
  author={Ashkboos, Saleh and Mohtashami, Amirkeivan and Croci, Maximilian L and Li, Bo and Cameron, Pashmina and Jaggi, Martin and Alistarh, Dan and Hoefler, Torsten and Hensman, James},
  journal={Advances in Neural Information Processing Systems},
  volume={37},
  pages={100213--100240},
  year={2024}
}

@inproceedings{
liu2024kivi,
title={{KIVI}: A Tuning-Free Asymmetric 2bit Quantization for {KV} Cache},
author={Zirui Liu and Jiayi Yuan and Hongye Jin and Shaochen Zhong and Zhaozhuo Xu and Vladimir Braverman and Beidi Chen and Xia Hu},
booktitle={Forty-first International Conference on Machine Learning},
year={2024},
url={https://openreview.net/forum?id=L057s2Rq8O}
}

@inproceedings{
hooper2024kvquant,
title={{KVQ}uant: Towards 10 Million Context Length {LLM} Inference with {KV} Cache Quantization},
author={Coleman Richard Charles Hooper and Sehoon Kim and Hiva Mohammadzadeh and Michael W. Mahoney and Sophia Shao and Kurt Keutzer and Amir Gholami},
booktitle={The Thirty-eighth Annual Conference on Neural Information Processing Systems},
year={2024},
url={https://openreview.net/forum?id=0LXotew9Du}
}

@article{zamirai2020revisiting,
  title={Revisiting bfloat16 training},
  author={Zamirai, Pedram and Zhang, Jian and Aberger, Christopher R and De Sa, Christopher},
  journal={arXiv preprint arXiv:2010.06192},
  year={2020}
}

@inproceedings{karimireddy2019error,
  title={Error feedback fixes signsgd and other gradient compression schemes},
  author={Karimireddy, Sai Praneeth and Rebjock, Quentin and Stich, Sebastian and Jaggi, Martin},
  booktitle={International conference on machine learning},
  pages={3252--3261},
  year={2019},
  organization={PMLR}
}

@inproceedings{
dao2024transformers,
title={Transformers are {SSM}s: Generalized Models and Efficient Algorithms Through Structured State Space Duality},
author={Tri Dao and Albert Gu},
booktitle={Forty-first International Conference on Machine Learning},
year={2024},
url={https://openreview.net/forum?id=ztn8FCR1td}
}

@inproceedings{
yang2025gated,
title={Gated Delta Networks: Improving Mamba2 with Delta Rule},
author={Songlin Yang and Jan Kautz and Ali Hatamizadeh},
booktitle={The Thirteenth International Conference on Learning Representations},
year={2025},
url={https://openreview.net/forum?id=r8H7xhYPwz}
}

@inproceedings{
quamba,
title={Quamba: A Post-Training Quantization Recipe for Selective State Space Models},
author={Hung-Yueh Chiang and Chi-Chih Chang and Natalia Frumkin and Kai-Chiang Wu and Diana Marculescu},
booktitle={The Thirteenth International Conference on Learning Representations},
year={2025},
url={https://openreview.net/forum?id=mnna9LUg7P}
}

@inproceedings{xu2025mambaquant,
  title={Mambaquant: Quantizing the mamba family with variance aligned rotation methods},
  author={Yue, Yuxuan and Hu, Xing and Yang, Dawei and Yuan, Zhihang and Jiang, Zixu and Chen, Zhixuan and Yu, Jiangyong and Zhou, Sifan},
  booktitle={International Conference on Learning Representations},
  volume={2025},
  pages={33231--33250},
  year={2025}
}

@inproceedings{
quamba2,
title={Quamba2: A Robust and Scalable Post-training Quantization Framework for Selective State Space Models},
author={Hung-Yueh Chiang and Chi-Chih Chang and Natalia Frumkin and Kai-Chiang Wu and Mohamed S. Abdelfattah and Diana Marculescu},
booktitle={Forty-second International Conference on Machine Learning},
year={2025},
url={https://openreview.net/forum?id=Zm0Kper4yx}
}

@inproceedings{tianqi-etal-2025-q,
  title={Q-Mamba: Towards more efficient Mamba models via post-training quantization},
  author={Tianqi, Chen and Chen, Yuanteng and Wang, Peisong and Xu, Weixiang and Zhu, Zeyu and Cheng, Jian},
  booktitle={Findings of the Association for Computational Linguistics: ACL 2025},
  pages={10594--10610},
  year={2025}
}

@inproceedings{
yang2024gated,
title={Gated Linear Attention Transformers with Hardware-Efficient Training},
author={Songlin Yang and Bailin Wang and Yikang Shen and Rameswar Panda and Yoon Kim},
booktitle={Forty-first International Conference on Machine Learning},
year={2024},
url={https://openreview.net/forum?id=ia5XvxFUJT}
}

@article{sun2023retentive,
  title={Retentive network: A successor to transformer for large language models},
  author={Sun, Yutao and Dong, Li and Huang, Shaohan and Ma, Shuming and Xia, Yuqing and Xue, Jilong and Wang, Jianyong and Wei, Furu},
  journal={arXiv preprint arXiv:2307.08621},
  year={2023}
}

@inproceedings{schlag2021linear,
  title={Linear transformers are secretly fast weight programmers},
  author={Schlag, Imanol and Irie, Kazuki and Schmidhuber, J{\"u}rgen},
  booktitle={International conference on machine learning},
  pages={9355--9366},
  year={2021},
  organization={PMLR}
}

@inproceedings{katharopoulos2020transformers,
  title={Transformers are rnns: Fast autoregressive transformers with linear attention},
  author={Katharopoulos, Angelos and Vyas, Apoorv and Pappas, Nikolaos and Fleuret, Fran{\c{c}}ois},
  booktitle={International conference on machine learning},
  pages={5156--5165},
  year={2020},
  organization={PMLR}
}

@article{chandiramani2026nemotron,
  title={Nemotron 3 super: Open, efficient mixture-of-experts hybrid mamba-transformer model for agentic reasoning},
  author={Chandiramani, Aakshita and Blakeman, Aaron and Olaoye, Abdullahi and Gupta, Abhibha and Somasamudramath, Abhilash and Khattar, Abhinav and Adesoba, Adeola and Renduchintala, Adi and Asif, Adil and Agrawal, Aditya and others},
  journal={arXiv preprint arXiv:2604.12374},
  year={2026}
}

@article{nagel2021white,
  title={A white paper on neural network quantization},
  author={Nagel, Markus and Fournarakis, Marios and Amjad, Rana Ali and Bondarenko, Yelysei and Van Baalen, Mart and Blankevoort, Tijmen},
  journal={arXiv preprint arXiv:2106.08295},
  year={2021}
}

@article{merity2016pointer,
  title={Pointer sentinel mixture models},
  author={Merity, Stephen and Xiong, Caiming and Bradbury, James and Socher, Richard},
  journal={arXiv preprint arXiv:1609.07843},
  year={2016}
}

@article{vaswani2017attention,
  title={Attention is all you need},
  author={Vaswani, Ashish and Shazeer, Noam and Parmar, Niki and Uszkoreit, Jakob and Jones, Llion and Gomez, Aidan N and Kaiser, {\L}ukasz and Polosukhin, Illia},
  journal={Advances in neural information processing systems},
  volume={30},
  year={2017}
}

@article{micikevicius2017mixed,
  title={Mixed precision training},
  author={Micikevicius, Paulius and Narang, Sharan and Alben, Jonah and Diamos, Gregory and Elsen, Erich and Garcia, David and Ginsburg, Boris and Houston, Michael and Kuchaiev, Oleksii and Venkatesh, Ganesh and others},
  journal={arXiv preprint arXiv:1710.03740},
  year={2017}
}

@article{penedo2024fineweb,
  title={The fineweb datasets: Decanting the web for the finest text data at scale},
  author={Penedo, Guilherme and Kydl{\'\i}{\v{c}}ek, Hynek and Lozhkov, Anton and Mitchell, Margaret and Raffel, Colin and Von Werra, Leandro and Wolf, Thomas and others},
  journal={Advances in Neural Information Processing Systems},
  volume={37},
  pages={30811--30849},
  year={2024}
}

@article{loshchilov2017decoupled,
  title={Decoupled weight decay regularization},
  author={Loshchilov, Ilya and Hutter, Frank},
  journal={arXiv preprint arXiv:1711.05101},
  year={2017}
}

@article{arora2024just,
  title={Just read twice: closing the recall gap for recurrent language models},
  author={Arora, Simran and Timalsina, Aman and Singhal, Aaryan and Spector, Benjamin and Eyuboglu, Sabri and Zhao, Xinyi and Rao, Ashish and Rudra, Atri and R{\'e}, Christopher},
  journal={arXiv preprint arXiv:2407.05483},
  year={2024}
}

@inproceedings{
Rae2020Compressive,
title={Compressive Transformers for Long-Range Sequence Modelling},
author={Jack W. Rae and Anna Potapenko and Siddhant M. Jayakumar and Chloe Hillier and Timothy P. Lillicrap},
booktitle={International Conference on Learning Representations},
year={2020},
url={https://openreview.net/forum?id=SylKikSYDH}
}

@article{arora2023language,
  title={Language models enable simple systems for generating structured views of heterogeneous data lakes},
  author={Arora, Simran and Yang, Brandon and Eyuboglu, Sabri and Narayan, Avanika and Hojel, Andrew and Trummer, Immanuel and R{\'e}, Christopher},
  journal={arXiv preprint arXiv:2304.09433},
  year={2023}
}

@inproceedings{lockard2019openceres,
  title={Openceres: When open information extraction meets the semi-structured web},
  author={Lockard, Colin and Shiralkar, Prashant and Dong, Xin Luna},
  booktitle={Proceedings of the 2019 Conference of the North American Chapter of the Association for Computational Linguistics: Human Language Technologies, Volume 1 (Long and Short Papers)},
  pages={3047--3056},
  year={2019}
}

@article{kwiatkowski2019natural,
  title={Natural questions: a benchmark for question answering research},
  author={Kwiatkowski, Tom and Palomaki, Jennimaria and Redfield, Olivia and Collins, Michael and Parikh, Ankur and Alberti, Chris and Epstein, Danielle and Polosukhin, Illia and Devlin, Jacob and Lee, Kenton and others},
  journal={Transactions of the Association for Computational Linguistics},
  volume={7},
  pages={453--466},
  year={2019},
  publisher={MIT Press One Rogers Street, Cambridge, MA 02142-1209, USA journals-info~…}
}

@inproceedings{rajpurkar2018know,
  title={Know what you don’t know: Unanswerable questions for SQuAD},
  author={Rajpurkar, Pranav and Jia, Robin and Liang, Percy},
  booktitle={Proceedings of the 56th Annual Meeting of the Association for Computational Linguistics (Volume 2: Short Papers)},
  pages={784--789},
  year={2018}
}

@inproceedings{joshi2017triviaqa,
  title={Triviaqa: A large scale distantly supervised challenge dataset for reading comprehension},
  author={Joshi, Mandar and Choi, Eunsol and Weld, Daniel S and Zettlemoyer, Luke},
  booktitle={Proceedings of the 55th Annual Meeting of the Association for Computational Linguistics (Volume 1: Long Papers)},
  pages={1601--1611},
  year={2017}
}

@inproceedings{dua2019drop,
  title={DROP: A reading comprehension benchmark requiring discrete reasoning over paragraphs},
  author={Dua, Dheeru and Wang, Yizhong and Dasigi, Pradeep and Stanovsky, Gabriel and Singh, Sameer and Gardner, Matt},
  booktitle={Proceedings of the 2019 Conference of the North American Chapter of the Association for Computational Linguistics: Human Language Technologies, Volume 1 (Long and Short Papers)},
  pages={2368--2378},
  year={2019}
}

@article{gao2021framework,
  title={A framework for few-shot language model evaluation},
  author={Gao, Leo and Tow, Jonathan and Biderman, Stella and Black, Sid and DiPofi, Anthony and Foster, Charles and Golding, Laurence and Hsu, Jeffrey and McDonell, Kyle and Muennighoff, Niklas and others},
  journal={Zenodo},
  year={2021}
}

@inproceedings{paperno2016lambada,
  title={The LAMBADA dataset: Word prediction requiring a broad discourse context},
  author={Paperno, Denis and Kruszewski, Germ{\'a}n and Lazaridou, Angeliki and Pham, Ngoc-Quan and Bernardi, Raffaella and Pezzelle, Sandro and Baroni, Marco and Boleda, Gemma and Fern{\'a}ndez, Raquel},
  booktitle={Proceedings of the 54th annual meeting of the association for computational linguistics (volume 1: Long papers)},
  pages={1525--1534},
  year={2016}
}

@inproceedings{welbl2017crowdsourcing,
  title={Crowdsourcing multiple choice science questions},
  author={Welbl, Johannes and Liu, Nelson F and Gardner, Matt},
  booktitle={Proceedings of the 3rd Workshop on Noisy User-generated Text},
  pages={94--106},
  year={2017}
}

@inproceedings{roemmele2011choice,
  title={Choice of Plausible Alternatives: An Evaluation of Commonsense Causal Reasoning.},
  author={Roemmele, Melissa and Bejan, Cosmin Adrian and Gordon, Andrew S},
  booktitle={AAAI spring symposium: logical formalizations of commonsense reasoning},
  pages={90--95},
  year={2011}
}

@inproceedings{zellers2019hellaswag,
  title={Hellaswag: Can a machine really finish your sentence?},
  author={Zellers, Rowan and Holtzman, Ari and Bisk, Yonatan and Farhadi, Ali and Choi, Yejin},
  booktitle={Proceedings of the 57th annual meeting of the association for computational linguistics},
  pages={4791--4800},
  year={2019}
}

@article{halko2011finding,
  title={Finding structure with randomness: Probabilistic algorithms for constructing approximate matrix decompositions},
  author={Halko, Nathan and Martinsson, Per-Gunnar and Tropp, Joel A},
  journal={SIAM review},
  volume={53},
  number={2},
  pages={217--288},
  year={2011},
  publisher={SIAM}
}
